\documentclass{article}

\usepackage{microtype}
\usepackage{graphicx}
\usepackage{subfigure}
\usepackage{pdfpages}
\usepackage{booktabs} 
\usepackage{tikz}
\usepackage{hyperref}

\usepackage[accepted]{icml2018}
\usepackage{comment}
\usepackage{amsmath,amssymb,algorithm,algorithmic,amsthm}
\usepackage{pdfsync,enumitem,dsfont}
\usepackage{tikz}
\newcommand*\circled[1]{\tikz[baseline=(char.base)]{\node[shape=circle,draw,inner sep=2pt] (char) {#1};}}
\newtheorem{thm}{Theorem}
\DeclareMathOperator{\diag}{diag} 
\newtheorem{lem}{Lemma}

\def\eqdef {\buildrel \rm def \over =}

\def\one{{\bf 1}}

\def\Re{\mathbb{R}}
\def\A{\mathcal{A}}
\def\S{\mathcal{S}}

\def\eqdef {\buildrel \rm def \over =}
\long\def\comment#1{}

\def\p{\phi}

\def\grad{\nabla}

\begin{document}

\twocolumn[

\icmltitle{Convergent Actor-Critic Algorithms Under Off-Policy Training and Function Approximation}




{\bf Hamid Reza Maei}\\
Criteo Research, Palo Alto, CA 94301, USA.\\
\href{mailto:h.maei@criteo.com}{h.maei@criteo.com}

%
%

\vskip 0.3in
]




\begin{abstract}
We present the first class of policy-gradient algorithms that work with both state-value and policy function-approximation, and are guaranteed to converge under off-policy training. Our solution targets problems in reinforcement learning where the action representation adds to the-curse-of-dimensionality; that is, with continuous or large action sets, thus making it infeasible to estimate state-action value functions (Q functions). Using state-value functions helps to lift the curse and as a result naturally turn our policy-gradient solution into classical Actor-Critic architecture whose Actor uses state-value function for the update. Our algorithms, Gradient Actor-Critic and Emphatic Actor-Critic, are derived based on the exact gradient of averaged state-value function objective and thus are guaranteed to converge to its optimal solution, while maintaining all the desirable properties of classical Actor-Critic methods with no additional hyper-parameters. To our knowledge, this is the first time that convergent off-policy learning methods have been extended to  classical Actor-Critic methods with function approximation. 

\end{abstract}

\section{Introduction and Related Works}

One of the most important desirable features of a Reinforcement Learning (RL) algorithm is the ability to learn off-policy. Off-policy learning refers to learning about a (or multiple) desirable policy (policies)  while the agent acts according to its own behavior policy, which may involve exploration. Off-policy learning is important because it allows the agent to learn about an optimal policy while it is exploring. It is also important for the case of off-line learning: for example, for the case of recommendation systems, we would like to learn about a better recommendation strategy than the one used previously, or conduct off-line A/B testing to avoid costs. Whether the learning is happening online or offline, freeing these two policies from each other makes the algorithms modular and easier to implement. For example, Q-learning is an off-policy learning because the agent can learn about a greedy policy while it could follow an exploratory policy. However, Q-learning has some limitations including the requirement for limited number of actions. Policy-gradient methods, on the other hand, are suitable to use for continuous actions (Williams, 1987; Sutton et~.al, 1999). Reinforce (Williams, 1987) is one of the most popular policy-gradient methods, however, the learning only can be done on-policy. In addition, the learning agent should wait until it collects all the rewards and then update. There has been attempts to make an off-policy version of Reinforce but the algorithm suffers from huge variance, particularly when the time-horizon is large or infinite (Tang \& Abbeel, 2010), as each reward signal must be multiplied by the {\em products} of importance ratios (Precup et~al., 2001). Temporal-Difference (TD) Learning methods solve this problem but they have been used for value function based methods (Sutton et~al., 2009; Maei et~al., 2010; Maei, 2011; Sutton et~al., 2016).  The classical Actor-Critic Architectures (Barto et~al., 1983; Sutton, 1984; Sutton et~al., 1999) provide an intuitive framework which combine both state-value function and policy-gradient ideas as part of Critic and Actor, respectively.  

The Q-Prop algorithm (Gu et~al., 2017)  uses Actor-Critic architecture for off-policy learning but  it uses state-action value functions (known as Q-functions) that requires representations for both state and actions, which implies a significant number of learning parameters (specially for continuous actions) making it potential for the curse-of-dimensionality/overfitting. The off-policy Actor-Critic algorithm proposed in Degris. et~al. (2012) uses state-value functions to update the Actor. The Critic uses  the GTD($\lambda$) algorithm (Maei, 2011) to estimate an off-policy evaluation for state-value function which will be used in Actor, and is one of the first attempts to solve the classical problem of Actor-Critic with off-policy learning.  The algorithm, has all the desirable features which we are seeking in this paper, except the fact that the actor-update is not based on the true gradient direction of the proposed objective function with linear value-function approximation~\footnote{See the last page, B. Errata,  in Degris et~al. (2012), with the following link: \url{https://arxiv.org/pdf/1205.4839.pdf}.}, 

In this paper, we solve this problem and propose the first convergent off-policy actor-critic algorithms, Gradient Actor-Critic and Emphatic Actor-Critic, with the following desirable features: online, incremental updating, linear complexity both in terms of memory and per-time-step computation, without adding any new hyper-parameter.  Our approach provides the first systematic solution that extends the classical actor-critic algorithms, originally developed for on-policy learning, to off-policy learning.

\section{RL Setting and Notations}
We consider standard RL setting where the learning agent interacts with a complex, large-scale, environment that has Markov Decision Process (MDP) properties. The MDP model is represented by quadruple $(\S,\A, \mathbf{R},\mathbf{P})$, where $\S$ denotes a finite state set, $\A$ denotes a finite or infinite action set, $\mathbf{R}=(r(s,a,s'))_{s,s' \in \S, a \in \A}$ and $\mathbf{P}=(P(s'|s,a))_{s,s' \in \S, a\in \A}$ denote real-valued reward functions and transition probabilities, respectively; for taking action $a$ from state $s$ and arriving at state $s'$.
 
The RL agent learns from data, generated as a result of interacting with the MDP environment. At time $t$ the agent takes $a_t \sim \pi(.|s_t)$, where $\pi: \S \times \A \rightarrow [0,1]$ denotes policy function, then the environment puts the agent in state $s_{t+1}$ with  the reward of $r_{t+1}$. As a result, the data is generated in the form of a {\em trajectory} and each sample point (fragment of experience) at time $t$ can be represented by tuple $(s_t,a_t,r_t,s_{t+1})$.

The objective of the agent is to find a policy function that has the highest amount of return in the long run; that is the sum of discounted future rewards.  Formally, the objective is to find the optimal policy $\pi^*= \arg \max_{\pi} V^{\pi}(s)$, $\forall s\in \mathcal{S}$, where $V^{\pi}(s) = \mathbb{E}_{\pi}[r_{t+1}+\gamma r_{t+2}+...| s_t=s]$, is called {\em state-value function} under policy $\pi$, with discount factor $\gamma\in [0,1)$, and $\mathbb{E}_{\pi}[.]$ 
represents expectation over random variable (data) generated according to the execution of policy $\pi$. From now on, by value-function, we always mean state-value function and we drop the subscripts from expectation terms.

Let us, represent all value-functions in a single vector, $V \in \Re^{|\S|}$, whose $s^{\rm th}$ element is $V(s)$, $T^{\pi}$ denotes Bellman Operator, defined as $T^{\pi}V \eqdef R^{\pi}+\gamma P^{\pi}V$, where $P^{\pi}$ denotes state-state transition probability matrix, where $P^{\pi}(s'|s) = \sum_{a} \pi(a|s) P(s'|s,a)$. Just for the purpose of clarify, and with a slight abuse of notations, we denote $\sum_{a}$ as the sum (or integral) over actions for both discrete and continuous actions (instead of using the notation $\int da$). Under MDP assumptions, the solution of $V=T^{\pi} V$, is unique and equal to $V=V^{\pi}$. 

In real-world large-scale problems, the number of states is too high. For example, for the case of $19\times19$ Computer Go we have roughly around $10^{170}$ states. This implies, we would need to estimate the value-function $V(s)$ for each state $s$ and without generalization; that is, function approximation, we are subject to  the {\em curse-of-dimensionality}.

To be practical, we would need to do function approximation. To do this, we can represent the state $s$ by a feature vector $\phi(s) \in \Re^{n}$. For example, for the case of $19\times19$ Computer Go we are shrinking a binary feature-vector of size $|\S| = 10^{170}$ (tabular features)  to the size of
$n=10^{6}$, and then we can do linear function approximation (linear in terms of learning parameters and not states). Now the value-function $V^{\pi}(s)$ can be approximated by $V_\theta(s)=\theta^\top \phi(s)$ and our first goal is to learn the parameter $\theta$ such that
$V_{\theta} \approx V^{\pi}$. For our notation, each sample (from the experience trajectory), at time $t$, is perceived in the form of tuple ${(\phi(s_t),a_t,r_{t+1},\phi(s_{t+1}))}_{t \ge 0}$, where for simplicity we have adopted the notation $\phi_t=\phi(s_{t+1})$.
 
Now we would need to do policy improvement given value-function estimate $V_{\theta}$ for a given policy $\pi$. Again to tackle the curse-of-dimensionality, for policy functions, we can parameterize the policy $\pi$ as $\pi_w$, where $w \in \Re^K$, where $K=O(n)$. Finally, through an iterative policy-improvement approach, we would like to converge to $\theta^*$, such that $V_{\theta^*} \approx V^{\pi^*}$.

The Actor-Critic (AC) approach is the only known method that allows us to use state-value functions in the the policy-improvement step, while incrementally updating value functions through Temporal-Difference (TD) learning methods, such as TD($\lambda$) (Sutton, 1998; 1988). This incremental, online update make the AC methods increasingly desirable.

Now let us consider the off-policy scenario, where the agent interacts with a fixed MDP environment and with a fixed {\em behavior} policy $\pi_b$. The agent would like to estimate the value of a given parametrized {\em target} policy $\pi_w$ and eventually find which policy is the best. This type of evaluation of $\pi_w$, from data generated according to a different policy $\pi_b$, is called {\em off-policy evaluation}. By defining the importance-weighting ratio $\rho_t=\frac{\pi_w(a_t|s_t)}{\pi_b(a_t|s_t)}$ and under standard MDP assumptions, statistically we can write the value-function $V^{\pi_w}$, in statistical form of
\begin{eqnarray*}
V^{\pi_w}(s) &=& \mathbb{E}[\rho_t r_{t+1}+\gamma \rho_t \rho_{t+1}  r_{t+2}+...| s_t=s]\\
    &=&  \mathbb{E}[\rho_t \left( r_{t+1}+\gamma  V^{\pi}(s_{t+1} \right)| s_t=s]\\
    &\approx& \theta^\top \phi(s).
\end{eqnarray*}
Again, for large-scale problems, we do TD-learning in conjunction with linear function approximation to estimate the $\theta$ parameters. Just like TD($\lambda$) with linear function approximation, which is used for  on-policy learning, the GTD($\lambda$) algorithm (Maei, 2011) and Emphatic-TD($\lambda$) (Sutton et~al., 2016) can be used for the problem of off-policy learning, with convergence guarantees. 

The question we ask is: What would be the Actor-update; that is, policy-improvement step? Particularly, we would like do a gradient ascent on policy-objective function such that the $w$-weights update, in expectation, exactly follow the direction of the gradient function.

\section{The Problem Formulation}
Degris et al. (2012)  introduced  the following objective function with linear value-function approximation,
\begin{eqnarray}
\label{eq:J}
J(w) = \sum_s d(s) V^{\pi_w}_\theta (s)= \theta^\top \mathbb{E}[\phi_t],
\end{eqnarray}
where $d(s)$ represents the stationary distribution for visiting state $s$ according to the behavior policy $\pi_b$. Please note, $\theta$ is an implicit function of $w$, since $V_{\theta}$ is an approximate estimator for $V^{\pi_w}$.  The goal is to maximize $J(w)$ by updating the policy parameters, iteratively,  along the gradient direction of $\nabla J(w)$, where $\nabla$ is gradient (operation vector) w.r.t policy parameters $w$ 

Degris et al. (2012) Off-Policy Actor-Critic Algorithm, Off-PAC, uses GTD($\lambda$) as Critic, however, the  Actor-update, in expectation, does not follow the true gradient direction of $J(w)$, thus questioning the convergence properties of Off-PAC. (See Footnote 1.)

In this paper, we solve this problem and derive an $O(n)$ convergent Actor-Critic algorithms based on TD-learning for the problem of off-policy learning, whose Actor-update, in expectation, follows the true gradient direction of $J(w)$, thus maximizing the $J(w)$.
%

In the next section, we discuss about the off-policy evaluation step. We discuss about two solutions; that is, GTD($\lambda$) (Maei, 2011) and Emphatic-TD($\lambda$) (Sutton et~.al, 2016). 

\section{Value-Function Approximation: GTD($\lambda$) and Emphatic-TD($\lambda$)-Solutions}
We consider linear function approximation, where $V_\theta=\Phi \theta$ and $\Phi \in \Re^{{|\cal S|} \times n}$ is a feature matrix  whose $s^{\rm th}$ row is the feature vector ${\phi(s)}^\top$. To construct the space in $\Re^{n}$ the column vectors of $\Phi$ need to be linearly independent and we make this assumption throughout the paper. We also assume that the feature vectors $\phi(s)$, $\forall s$ all have a {\em unit feature-value of $1$} in their $n^{\rm th}$ element. This is not needed if we use tabular features, but for the case of function approximation we will see later why this unit feature-value is needed, which is typically used for linear (or logistic) regression problems as the intercept term.

To estimate $V_\theta$, it is  natural to find $\theta$ through minimizing the following mean-square-error objective function: 
\[
{\rm MSE}(\theta)=\sum_s d(s) \left[V^{\pi}(s)-\theta^\top \phi(s) \right]^2,
\]
where $d(s)$, $\forall s \in {\cal S}$, denotes the underlying state-distribution of data, which is generated according to $\pi_b$. The square-error is weighted by $d(s)$, which makes our estimation biased towards the distribution of data generated by following the policy $\pi_b$. The ideal weight would be the underlying stationary distribution under the target policy $\pi$. However, we claim that this is a natural fact in nature, and used in all supervised learning methods. There are ad-hoc methods to re-weight distributions if needed, but they are outside of the scope of this paper.

By minimizing the MSE objective function w.r.t $\theta$, we get
\begin{equation}
\label{eq:mse_solution}
V_\theta = \Pi_D V^{\pi},\quad \Pi_D = \Phi (\Phi^\top D \Phi)^{-1} \Phi^\top D 
\end{equation}
where $\Pi_D$ is projection operator,  and $D=\diag(d(1),\cdot,d({|\cal S|}))$. Also in this paper we assume that $d(s)>0$, $\forall \in \S$, meaning all states should be visited. We call this solution, {\em MSE-Solution}.

Historically, there are two alternative and generic solutions, called GTD($\lambda$)  (Maei, 2011) and Emphatic-TD($\lambda$) (Sutton, et~al., 2016). The MSE-Solution is that is the special case of the two solutions (when $\lambda=1$). Later we discuss about the merits of the two solutions and the reasons behind them.

\paragraph{GTD($\lambda$)-Solution: The Projected Bellman-Equation with Bootstrapping Parameter $\lambda$:}

To find the approximate solution of $V_\theta$ for evaluating the target policy $\pi$, historically, the classical projected Bellman-Equation (Bertsekas, Sutton et al. Maei, 2011) has been used for the problem of off-policy evaluation is, $V_{\theta}=\Pi_{D} T^{\pi, \lambda} V_{\theta}$,
where $\Pi_D$ is the projection operator defined in Eq.~\ref{eq:mse_solution}, and $(T^{\pi,\lambda}V)(s)= \sum_{a,s'} \pi(a|s)P(s'|s,a) [ r(s,a,s')+ \gamma (1-\lambda)V(s')+\lambda (T^{\pi,\lambda}V)(s') ]$. We can convert the above Matrix-vector products into the following statistical form (Maei, 2011):
%
\begin{equation}
\label{eq:tdl}
\mathbb{E}[\rho_t \delta_t e_t ]=0,
\end{equation}
\[
\delta_t=  r_{t+1}+\gamma \theta^\top \phi_{t+1}-\theta^{\top} \phi_{t+1},~~ \rho_t = \frac{\pi_w(a_t|s_t)}{\pi_b(a_t|s_t)},
\]
\[
e_t = \phi_t+\gamma \lambda \rho_{t-1} e_{t-1},
\]
where $\phi_t \eqdef \phi(s_t)$. (Note, following GTD($\lambda$) in Maei (2011), here we have done a small change of variable for $e$
without changing the solution.)

The GTD($\lambda$) is used to find the solution of Eq.~\ref{eq:tdl} with convergence guarantees. As such we call the fixed-point, GTD($\lambda$)-Solution. The GTD($\lambda$) main update is as follows:
\begin{equation}
\theta_{t+1}=\theta_t+\alpha_t \rho_t \left[ \delta_t e_t - \gamma (1-\lambda) \phi_{t+1} (e_t^\top u_t) \right],
\end{equation}
where $\alpha_t$ is step-size at time $t$, $u$ represents a secondary set of weights, updated according to $u_{t+1}=u_t+\alpha_t^{u} \left[\rho_t \delta_t e_t-(u_t^\top \phi_t) \phi_t \right]$, where $\alpha_t^u$ is a step-size at time $t$.

There are a few points to make regarding the solution of GTD($\lambda$): 
\begin{itemize}
\item  For the case of tabular features, or features that span the state-space, the solution is independent of $\lambda$ value and is equal to the true solution $V^{\pi}$.
\item For the case of linear function approximation the solution can depend on $\lambda$.
\item For $\lambda=1$ the solution is equivalent to MSE-solution; that is $\Pi_D V^{\pi}$. In addition, GTD(1)-update would not need a second set of weights and step-size, making it very simple, as we can see from its main update.
\item For the case of on-policy learning, GTD(1) and TD(1) are identical.
\end{itemize}

\paragraph{Emphatic-TD($\lambda$)-Solution:}

An alternative solution to the problem of off-policy prediction and its solution with function approximation is an {\em Emphatic} version of Projected Bellman-Equation, developed by Sutton et~al. (2016), in projection operator, $\Pi_D$, now we have $D \leftarrow DF$, where $F$ is an emphatic (positive definite) diagonal matrix. Later we will discuss about the properties of matrix $F$. In statistical form, the solution satisfies,
\begin{equation}
\label{eq:emphatic}
\mathbb{E}[\rho_t \delta_t e_t]=0, \quad e_t= m_t \phi_t+\gamma \lambda \rho_{t-1} e_{t-1},
\end{equation}
\[
m_t = 1 +\gamma \rho_{t-1} (m_{t-1}-\lambda),
\]
where $m_{-1}=\lambda$ and $m_t$ remains always strictly positive. (Please note, we have combined $m_t=\lambda+(1-\lambda) f_t$ and $f_t=1+\gamma \rho_{t-1} f_{t-1}$ used in Sutton et~al. (2016). Also note, since the form of the update for both GTD($\lambda$) and Emphatic-TD($\lambda$) looks the same, due to simplicity, we have used the same notation for the eligibility trace vector, $e$.)

The Emphatic-TD($\lambda$) update is as follows:
\begin{equation}
\theta_{t+1}=\theta_t+\alpha_t \rho_t \delta_t e_t,
\end{equation}
where $e_t$ follows Eq.~\ref{eq:emphatic} update.

Here, we make a few points regarding the solution of Emphatic-TD($\lambda$):
\begin{itemize}
\item For the case of tabular features, or features that can span the state-space, the Emphatic-TD($\lambda$) and GTD($\lambda$) both converge to the true solution $V^{\pi}$.
\item For the case of linear function approximation, both Emphatic-TD($\lambda$) and GTD($\lambda$) solutions depend on $\lambda$
but they may differ: 
\begin{itemize}
\item For the case on-policy learning  both solutions  are the same, because $F$ would become a constant diagonal matrix
\item  For the case of off-policy learning, the solutions will differ, and still it is not clear which solution has a better solution advantage.
\end{itemize}
\item  Both Emphatic-TD(1) and GTD(1) have the same update (identical, as we have $m_t=1$ for all $t$) and converge to the MSE-Solution.
\end{itemize}
It is worth to discuss MSE-Solution here as both Emphatic-TD(1) and GTD(1) become identical and converge to MSE-Solution.  The question is why not using MSE-Solution, with $\lambda=1$, and why to have such a variety of solutions based on bootsrapping parameter $\lambda$? The truth is, it is widely accepted that the main reason bias-variance trade-off. In fact MSE-Solution is natural solution, but when $\lambda=1$, $e$ traces become large and cuase a huge variance around the fixed-point, thus may result an inferior solution. (This is also known as Monte-Carlo solution in Sutton \& Barto, 1998.)  However, there has been little investigation on  how by reducing the variance of $e$ traces one can reduce the overall variance and thus converge to a quality solution. We did some simple experiments and by normalizing the eligibility trace vector, e.g. $e/|e|$, we were able to get superior results as is shown in Fig.~\ref{fig:mse}. (For the details of the exepriment see  the 19-state random walk in Sutton \& Barto, 1998.) 

\begin{figure}[h!]
\centering
  \includegraphics[scale=0.23]{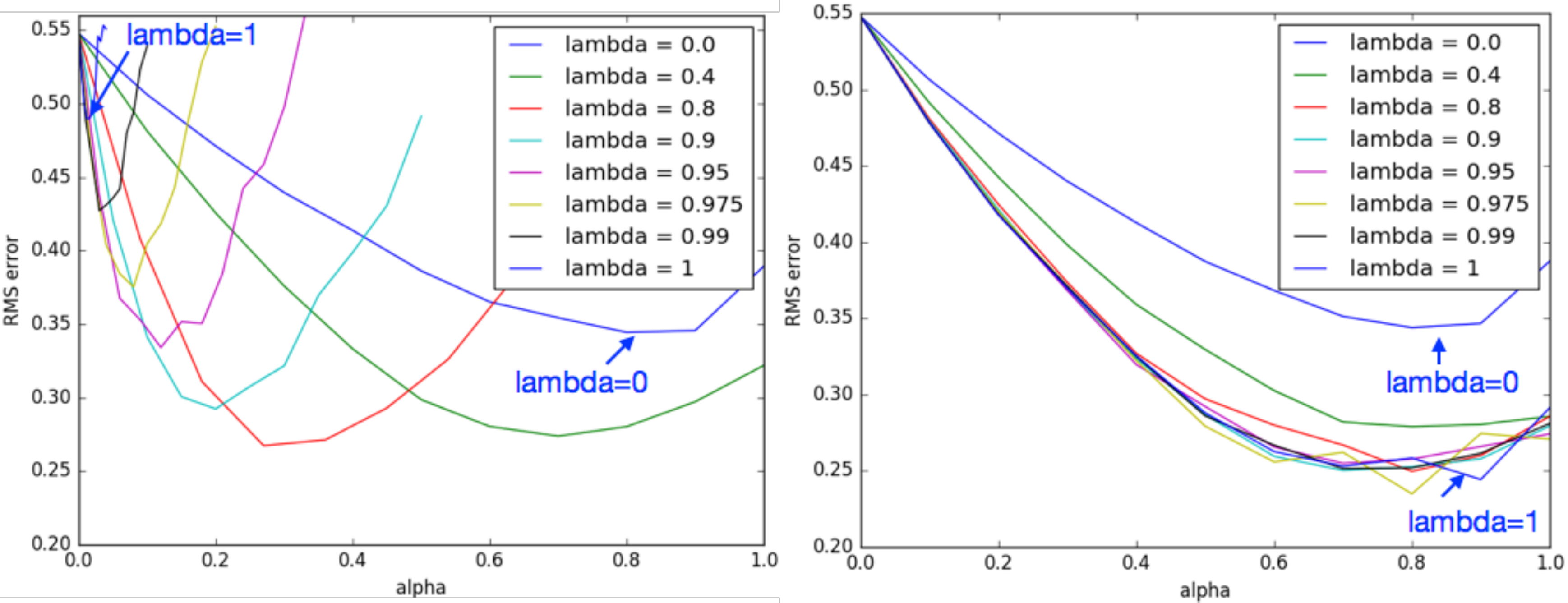}
  \caption{Parameter studies: RMS values (vertical axis) with various step-size, alpha (x-axis) and $\lambda$ values. The right panel normalizes the eligibility traces which leads to the best value for $\lambda=1$, which is also slightly better than the best value (with $\approx \lambda=0.8$) in the left pannel. }
  \label{fig:mse}
\end{figure}

\section{The Gradient Direction of $J(w)$ }
In this section, we explicitly derive the exact gradient direction of the function $J(w)$ in Eq.~\ref{eq:J} w.r.t. policy parameter $w$. 

As we can see the GTD($\lambda$) solution in Eq.~\ref{eq:tdl} and Emphatic-TD($\lambda$) solution in Eq.~\ref{eq:emphatic} look similar. (Their solution is different due to a different update for $e$.) Thus, we provide the same form of gradient of $J(w)$ for both, as follows: First, we compute the gradient of $J(w)$ from Eq.~\ref{eq:J}, $\nabla J(w)=\frac{\partial {\theta}^\top}{\partial w}\mathbb{E}[\phi]$.  Now, to obtain the matrix, $\frac{\partial {\theta}^\top}{\partial w}$, we transpose Eq.~\ref{eq:tdl} (Eq.~\ref{eq:emphatic}) and then take its gradient as follows ($\nabla$ is column vector operator):
\begin{eqnarray}
\label{eq:td_grad}
0 &=& \nabla \mathbb{E}[\rho_t \delta_t e_t^\top]  \\
&=& \underbrace{\mathbb{E}[ \nabla \rho_t \delta_t e_t^\top]}_{\circled{1}} +\underbrace{\mathbb{E}[\rho_t \nabla \delta_t e_t^\top]}_{\circled{2}}+\underbrace{\mathbb{E}[\rho_t  \delta_t \nabla e_t^\top]}_{\circled{3}}, \nonumber
\end{eqnarray}
where we can show,
\[
\circled{1}= \mathbb{E}[\rho_t \delta_t \nabla \log{\pi_w(a_t|s_t)}e_t^\top],
\]
\[
\circled{2}=\frac{\partial {\theta}^\top}{\partial w}\mathbb{E}[\rho_t \left( \gamma \phi_{t+1}-\phi_t \right)e_t^\top].
\]
Putting all together and solving for the matrix  $\frac{\partial {\theta}^\top}{\partial w}$ we get,
\begin{eqnarray}
\label{eq:grad_theta}
\frac{\partial {\theta}^\top}{\partial w}&=&\left(\circled{1}+ \circled{3}\right) \mathbb{E}[\rho_t \left(\phi_t-\gamma \phi_{t+1} \right){e_t}^\top ]^{-1} \nonumber\\
&=&\mathbb{E}[\rho_t  \delta_t \left(\nabla \log{\pi(a_t|s_t)}  e_t^\top+\nabla e_t^\top \right)] {A(\lambda)^{\top}}^{-1},\nonumber \\
\end{eqnarray}
where
\begin{equation}
{A(\lambda)}^\top \eqdef \mathbb{E}[\rho_t \left(\phi_t-\gamma \phi_{t+1} \right){e_t}^\top ].
\end{equation}

$A(\lambda)$ matrix is invertible because we have assumed that the column vectors of $\Phi$ are linearly independent. This is a realistic assumption to construct a space in $\Re^{n}$ where the approximate solution is located.

Now we substitute the equalities in the Eq.~\ref{eq:grad_theta} to obtain $\frac{\partial {\theta}^\top}{\partial w}$, which will be used in $\nabla J(w)=\frac{\partial {\theta}^\top}{\partial w}\mathbb{E}[\phi_t]$, to obtain the exact gradient. To do this, first let us consider the following definitions:
\begin{equation}
\label{eq:etadef}
\eta(\lambda) \eqdef {A(\lambda)^{\top}}^{-1} \mathbb{E}[\phi_t],
\end{equation}
\begin{equation}
\label{eq:flt}
f_t^\lambda \eqdef e_t^\top \eta(\lambda),
\end{equation}
\begin{equation}
\label{eq:fl}
f^{\lambda}(s) \eqdef  \mathbb{E}[f_t^\lambda|s_t=s]=\mathbb{E}[e_t|s_t=s]^\top \eta(\lambda),
\end{equation}
and  $f^\lambda \in \Re^{|\S|}$, and also we define the diagonal matrix $F$ whose diagonal elements are vector $f$. Using definitions(\ref{eq:etadef},\ref{eq:fl}), and given the fact that conditioning on $s_t=s$ we get $s_{t-1}$ and $s_{t+1}$ independent, then we have the following equations: First, Eq.~\ref{eq:etadef} can be written as:

\begin{equation}
\label{eq:fleq}
\mathbb{E}[\rho_t \left(\phi_t-\gamma \phi_{t+1} \right) f^\lambda(s_t) ]=\mathbb{E}[\phi_t],
\end{equation}
and also we have,
\begin{eqnarray}
\label{eq:grad_general}
\nabla J(w)&=& \frac{\partial {\theta}^\top}{\partial w} \mathbb{E}[\phi_t] \nonumber \\
&=& \mathbb{E}[\rho_t \delta_t \left( f_t^\lambda \nabla \log{\pi_w(a_t|s_t)}  +\nabla {e_t}^\top \eta(\lambda) \right) ]. \nonumber  \\
\end{eqnarray}
The above gradient is the exact gradient direction in statistical form.

The most important question is to identify the values of  $f_t^\lambda$ and $\nabla {e_t}^\top \eta(\lambda)$, which would be our next task.  
We will show that these terms are state dependent and can not be ignored for the case of off-policy learning, unless the problem is on-policy.

\subsection{On-Policy scenario:} 

\begin{lem}\label{lem:on_policy}(On-Policy GTD($\lambda$) and Emphatic-TD($\lambda$) Solutions) For the problem of on-policy, where $\rho_t=1$, $\forall t$, the  Emphatic-TD($\lambda$),GTD($\lambda$), and TD($\lambda$) solutions are identical.
\end{lem}
\begin{proof}
From Eq.~\ref{eq:tdl}, it is clear that on-policy GTD($\lambda$) and TD($\lambda$) have the same solution. For the case of emphatic-TD($\lambda$), when $\rho_t=1$, $\forall t$, we get $\lim_{t \rightarrow \infty} m_t=(1-\gamma \lambda)/(1-\gamma)$. Thus we can use the convergent value, $m_{\infty}$ in the update for $e_t$; that is, $e_t=m_{\infty} \phi_t+\gamma \lambda e_{t-1}$. Let us divide the Eq.~\ref{eq:emphatic} by $m_{\infty}$, and define $\hat{e}_t = e_t/m_{\infty}$, then we get $\hat{e}_t= \phi_t+\gamma \lambda \hat{e}_{t-1}$, where $\hat{e}_{-1}=0$. Thus the expectation term becomes identical to Eq.~\ref{eq:tdl}, finishing the proof.
\end{proof}

\begin{thm}\label{thm:on_policy}($\nabla J(w)$ for the Problem of On-Policy  (aka TD($\lambda$)-Solution))
\begin{eqnarray}
(1-\gamma \lambda)\nabla J(w) = \mathbb{E}[ \delta_t \log{\pi_w(a_t|s_t)} ],
\end{eqnarray}
\end{thm}
\begin{proof}
The term $f^{\lambda}(s_t)$ satisfying in  Eq.~\ref{eq:fleq}, can be replaced by $f_t^\lambda$, which by definition is equivalent to $e_{t}^\top \eta(\lambda)$ and since the matrix $A$ is invertible (due to the assumption of linearly independent column feature vectors of $\Phi$), then $\eta$ solution is unique, which implies uniqueness of $f^{\lambda}(s)$ for a given $s$. Now, we show that  $f^{\lambda}=\frac{1}{1-\gamma \lambda} \one$, which is constant $\forall s$. Using $d^{\top}P^\pi=d^{\top}$ in Eq.~\ref{eq:fleq}, we get
\begin{flalign*}
&\mathbb{E}[\rho_t \left(\phi_t-\gamma \phi_{t+1} \right) f^\lambda(s) ] \\
&=\sum_{s,s'} d(s)\left( \phi(s)-P^\pi(s'|s) \phi(s')\right)\frac{1}{1-\gamma \lambda}   \\
&=\frac{1}{1-\gamma \lambda}\left[\sum_s d(s) \phi(s)- \sum_{s'} (\sum_s d(s)P^\pi(s'|s)) \phi(s') \right],\\
&= \mathbb{E}[\phi_t],
\end{flalign*}
where we have used $\sum_s d(s)P^\pi(s'|s)=d(s')$. Thus $f^{\lambda}(s)=\frac{1}{1-\gamma \lambda}$ is the solution for all $s$, and using Lemma (\ref{lem:on_policy}) , we finish the proof.
\end{proof}

\subsection{Off-Policy scenario:} 
Unlike on-policy, for the problem of off-policy $f^\lambda(s)$ is state $s$ dependent and is not constant (see the proof of Theorem(\ref{thm:on_policy})). Estimating the value of  $f^\lambda(s)$ will require estimating the parameter $\eta(\lambda)$
which comes with complexities in terms of computations. This is mainly due to the fact that we don't have on-policy criteria
$d^{\top}P^\pi=d^{\top}$. However, when $\lambda=1$; that is, for the case of GTD(1) solution, which is equivalent to MSE solution, we would be able to find the exact value of 
\[
f(s) \eqdef f^{\lambda=1}(s), ~ \forall s \in \S,
\]
which would enable us to do $O(n)$ sampling from the true gradient of $J(w)$. This is one of our main contributions in this paper.

\section{The Gradient Direction of $J(w)$ with GTD(1)-Solution}

\begin{lem}\label{lem:flvalue}(Value of $f$ for the Problem of Off-Policy with GTD(1)-Solution) 
For the problem of off-policy the value of $f$, defined as $f^{\lambda=1}$, satisfying in Eq.~\ref{eq:fleq} is
 
\begin{equation}
f(s) =  \lim_{t \rightarrow \infty} \mathbb{E}[f_t|s_t=s], \forall s
\end{equation}
where {\rm where}  $f_t  =1+\gamma \rho_{t-1} f_{t-1}$, $f_{-1}=0$, $t\ge 0$, $\rho_t=\frac{\pi(a_t|s_t)}{\pi_b(a_t|s_t)}$, and
all the elements of $\eta(1)$ are zero except the $n^{\rm th}$ element which is $1$. 
\end{lem}
\begin{proof}
Just like the proof of Theorem (\ref{thm:on_policy}),  let us assume that $f(s)$  is equal to $ \lim_{t \rightarrow \infty} \mathbb{E}[f_t|s_t=s]$. Due to having a unique solution, if it satisfies in Eq.~\ref{eq:fleq} then it must be the unique solution. Please note that here, we have done a slight abuse of notation and have used $f(s)$ which by definition is $\mathbb{E}[f_t^{\lambda=1}|s_t=s]$ and also for the iteration we have used $f_t$ with subscript $t$. These are two different variables and should not be mistaken, but for simplicity we have adopted this notation in the paper.

To do the proof, first, we can show that, the product of the diagonal matrix $D$ (whose diagonal element is $d$) times $f$ is $Df=(1-{P^{\pi}}^\top)^{-1}d$ (See supplementary materials). Let us define the diagonal matrix $F$ with diagonal elements of $f$, where $F {\bf 1} = f$, we have 
\begin{flalign*}
&\mathbb{E}[\rho_t \left(\phi_t-\gamma \phi_{t+1} \right) f(s) ] \\
&=\sum_{s,s'} d(s)f(s) \left( \phi(s)-P^\pi(s'|s) \phi(s')\right)   \\
&= \left[(I-\gamma P^\pi)\Phi \right]^\top DF {\bf 1} \\
&=\Phi^\top (1-\gamma {P^{\pi}}^\top)(1-\gamma{P^{\pi}}^\top)^{-1}d \\
&=\sum_s d(s) \phi(s)= \mathbb{E}[\phi_t].
\end{flalign*}
Again, given the fact that $\eta(1)$ has a unique solution, and the $n^{\rm th}$ element of all the feature vectors $\phi(s)$, $\forall s$, have a unit value of $1$, we can see that the $n^{\rm th}$ element of $e_t$ would be equivalent to $f_t$ and thus $\eta(1)$ with all zero elements except the $n^{\rm th}$ element with the value of $1$ satisfies in Eq.~\ref{eq:fleq}.
Thus, finishing the proof.
\end{proof}

\begin{thm}\label{thm:J_lambda1}($\nabla J(w)$ for the Problem of Off-Policy with GTD(1)-Solution (aka MSE-Solution))
\begin{eqnarray}
\label{thm:grad_trueJ}
\nabla J(w) = \lim_{t \rightarrow \infty} \mathbb{E}[ \rho_t  \delta_t \psi_t ],
\end{eqnarray}
where $\psi_t = f_t \nabla \log{\pi_w(a_t|s_t)}+ \gamma \rho_{t-1} \psi_{t-1}$,$f_t  =1+\gamma \rho_{t-1} f_{t-1}$, $\psi_{-1}$ is zero vector, $f_{-1}=0$  $\rho_t  =\frac{\pi(a_t|s_t)}{\pi_b(a_t|s_t)}$.
\end{thm}
\begin{proof}
The first term of $\nabla J(w)$ in Eq.~\ref{eq:grad_general}; that is $\mathbb{E}[\rho_t \delta_t  f_t^\lambda \nabla \log{\pi_w(a_t|s_t)}]$ can be written as $\sum_s d(s) \mathbb{E}[f_t^\lambda|s_t=s] \mathbb{E}[\rho_t \delta_t  \nabla \log{\pi_w(a_t|s_t)}|s_t=s]$ since, due to MDP properties, conditioning on $s_t=s$ makes $f_t^{\lambda=1}$ independent of $s_{t+1}$ as it depends on $s_{t-1}$ and past. Now from  $\mathbb{E}[f_t^{\lambda=1}|s_t=s]=f(s)$, we use the value of $f(s)$ from Lemma (\ref{lem:flvalue}), then we get
\begin{flalign*}
&\mathbb{E}[\rho_t \delta_t  f_t^\lambda \nabla \log{\pi_w(a_t|s_t)}] \\
&= \sum_s d(s) f(s) \mathbb{E}[\rho_t \delta_t \nabla \log{\pi_w(a_t|s_t)}|s_t=s] \\
&= \mathbb{E}[f_t \rho_t \delta_t \nabla \log{\pi_w(a_t|s_t)}].
\end{flalign*}
Now we turn into the second term which is $\mathbb{E}[\nabla {e_t}^\top \eta(\lambda)]$ with $\lambda=1$. From Lemma(\ref{lem:flvalue}), we see $\eta(1)$ is constant and zero except the last element which is one. Thus $\nabla {e_t}^\top \eta(1)$, can be written as $\nabla \left({e_t}^\top \eta(1) \right)=\nabla f_t$ since ${e_t}^\top \eta(1)=f_t$. Now by putting all together and using
$\nabla f_t=\gamma \rho_{t-1} \left(\nabla \log{\pi_w(a_{t-1}|s_{t-1})} f_{t-1}+ \nabla f_{t-1} \right) $ and by defining $\psi_t \eqdef f_t \nabla \log{\pi_w(a_t|s_t)}+\nabla f_t$, we get the recursive form of $\psi_t$, thus finishing the proof.
\end{proof}
Please note, for $\lambda<1$, finding a value for $f^\lambda$ that can be used in $\nabla J(w)$ and enabling us doing
$O(n)$ sampling  was not possible due to $\lambda$ parameter. However, later we will see that the Emphatic-TD($\lambda$) solution, solves this problem.

\section{Gradient-AC: A Gradient Actor-Critic with GTD(1) as Critic}
By $O(n)$ sampling from $\nabla J(w)$ in Theorem (\ref{thm:J_lambda1}), we present the first off-policy gradient Actor-Critic algorithm, that is convergent, in Table (\ref{tab:grad_ac}). The Critic uses GTD(1) update while the Actor uses policy-gradient method to maximize the $J(w)$ objective function. It is worth to mention that the complexity cost of the new algorithm is $O(n)$, the same as classical AC method, with no additional hyper/tuning parameters.

 \begin{algorithm}[h!]
  \caption{The Gradient-AC Algorithm}
\label{tab:grad_ac}
\begin{algorithmic}[1]
\STATE {\bfseries Initialize} $\rho$,$f$, $e$ and $\psi$, $w$, $\theta $ to zero values
\STATE {\bfseries Choose} proper step-size values for $\alpha$ and $\beta$
\REPEAT

\FOR{ each sample $(\phi,a,r,\phi')$ generated by $\pi_b$}
\STATE $ e \leftarrow  \p+ \gamma \rho e $
\STATE $f \leftarrow 1+\gamma \rho f$
\STATE $\psi \leftarrow f \nabla  \log{\pi_w(a|s)}+ \gamma \rho \psi$
\STATE $\rho \leftarrow \frac{\pi_w(a|s)}{\pi_b(a|s)}$
\STATE $\delta  \leftarrow r+\gamma \theta^{\top}\p'-\theta^{\top}\p$
\STATE $\theta \leftarrow \theta+\alpha \rho \delta e$
\STATE $w \leftarrow w+\beta \rho \delta \psi$


\ENDFOR
\UNTIL{$w$ converges}
\end{algorithmic}
\end{algorithm}

\subsection{Convergent Analysis of G-AC}
In this section we provide convergence analysis for the Gradient-AC algorithm. Since the Actor-update is based on true gradient direction, we use existing results in literature avoid repetitions. Before providing the main theorem, let us consider the following assumptions. The first set of assumptions are related to data sequence and parametrized policy:
\begin{enumerate}[label=(A\arabic{*})]
\item $a_t \sim \pi_b(.|S_t)$, such that (s.t.) $\pi_b$ is stationary and $\pi_b(a|s)>0$, $\forall (s,a) \in \S \times \A$;
\label{cond:datafirst}
\item The Markov processes $(s_t)_{t \ge 0}$ and $(s_t, a_t)_{t \ge 0}$ are in steady state, irreducible and aperiodic, with stationary distribution $\mu(s)>0$, $\forall s$;  
\item $(s_{t+1},r_{t+1}) \sim P(\cdot,\cdot|s_t,a_t)$, and $\exists C_0$ s.t. ${\rm Var}[r_{t}|s_t=s] \le C_0 $ holds almost surely (a.s.).
\label{cond:datalast}
\end{enumerate}
Assumptions on parametrized policy $\pi_w$ is as follows
\begin{enumerate}[label=(P\arabic{*})]
\item For every $(s,a)\in \S \times \A$ the mapping $w \mapsto \pi_w(a|s)$ is twice differentiable;
\item  $\sup_{w, (s,a) \in \S \times \A} \| \grad \log \pi_w(a|s) \| < \infty$ and $\grad \log \pi_w(a|s)$ has a bounded derivative  $\forall (s,a)\in \S \times \A$ . Note $\|.\|$ denotes Euclidean norm.
\end{enumerate}

\begin{enumerate}[label=(F\arabic{*})]
\item Features are bounded according to Konda \& Tsitsiklis (2003), and we follow the same noise properties conditions. 
\end{enumerate} 

For the convergence analysis we follow the two-time-scale  approach (Konda \& Tsitsiklis, 2003; Borkar, 1997; Borkar, 2008; Bhatnagar et~al., 2009).  We will use the following step-size conditions for the convergence proof:
\begin{enumerate}[label=(S\arabic{*})]
\item  $\sum_{t=0}^\infty \alpha_t = \sum_{t=0}^\infty \beta_t = +\infty$, $\sum_{t=0}^\infty (\alpha_t^2 + \beta_t^2) < +\infty$, and $\alpha_t / \beta_t \rightarrow 0$.
\label{cond:sslast}
\end{enumerate}

The actor update  can be written  in the following form:
\begin{equation}
\label{eq:omg_iter}
w_{t+1} = \Gamma\big(w_{t}+\beta_{t}(f\,(w_t,\theta_t) + M_{t+1})\big),
\end{equation}
where 
\[
f(w_t,\theta_t) = \mathbb{E}[\rho_t \delta_t \psi_t|w_t,\theta_t], \;
 M_{t+1}= \rho_t \delta_t \psi_t- f(w_t,\theta_t).
\]
Note, $\Gamma$ projects its argument to a compact set ${\cal C}$ with smooth boundary, that is, if the iterates leaves ${\cal C}$ it is projected to the closest or some convenient point in ${\cal C}$, that is,  $\Gamma(w) = \arg\min_{w'\in \cal C} \| w'-w\|$. Here, we choose ${\cal C}$ to be the largest possible compact set. 
 We consider ordinary differential equation (ODE) approach for the convergence of our proof and show our algorithm converges to the set of all asymptotically stable solution of the following ODE
\begin{equation}
\label{eq:ode_omg}
\dot{w} = \hat{\Gamma} (\grad J\, (w)), \quad w(0)\in {\cal C},
\end{equation} 
where  $
	\hat{\Gamma} (f \,(w))
	 \eqdef \lim_{0 < \varepsilon \rightarrow 0} 
		\frac{\Gamma \big(w + \varepsilon \,f(w)\big) -w}{\varepsilon}$.  Note,  if $w \in {\cal C}^\circ$ we have $\hat{\Gamma}(f\, (w)) = f\,(w)$,  otherwise, $\hat{\Gamma} (f\,(w)) $ projects  $f\,(w)$ to the tangent space of $\partial {\cal C}$ at $\Gamma(f\,(w))$.
\begin{thm}\label{thm:main}(Convergence of Gradient-AC)
Under the conditions listed in this section, as $t \rightarrow +\infty$,   $w_t$ converges to the set of all asymptotically stable solution of (\ref{eq:ode_omg}) with probability 1.
\end{thm}
\begin{proof}
The proof exactly follows on a two-timescales, in steps-size, convergence analysis. We use the results in Borkar (1997; 2008, see Lemma 1 and Theorem 2, page 66. Also see, Bhatnagar et~al., 2009; Konda \& Tsitsiklis; 2003). The expected update of the actor, according to Theorem (\ref{thm:J_lambda1}) is exactly $\nabla J(w)$ and also the critic, GTD(1), is a true gradient method. As such the proof will follow and for brevity, we have omitted the repetition of the proof.   
\end{proof}

\section{$\nabla J(w)$ for Emphatic-TD($\lambda$) solution}
With Emphatic-TD($\lambda$) solution, now we aim to optimize its corresponding $J(w)$ objective. Again one can show that $\eta(\lambda)$ with all zero elements except the $n^{\rm th}$ element with value of 1, will satisfy in $\eta(\lambda)$ equation (see Eq.\ref{eq:fl} and Eq.\ref{eq:fleq} ) and as a result, Eq.~\ref{eq:grad_general} becomes
\begin{eqnarray}
\nabla J(w)= \mathbb{E}[\rho_t \delta_t \left( f_t^\lambda \nabla \log{\pi_w(a_t|s_t)}  +\nabla f_t^\lambda \right) ].
\end{eqnarray}

\begin{lem}(The value of $f^\lambda$ for Emphatic-TD($\lambda$) Solution) 
Given Emphatic-TD($\lambda$) solution, from equation satisfying $f^\lambda$ in Eq.\ref{eq:fleq}, we have
\[
f^\lambda_t= m_t+\gamma \lambda \rho_{t-1} f^\lambda_{t-1},
\]
where $m_t = 1 +\gamma \rho_{t-1} (m_{t-1}-\lambda)$ and $m_{-1}=\lambda$.
\end{lem}
\begin{proof}
The proof follows derivations in Lemma (\ref{lem:flvalue}) and due to limitation of space, we leave the main calculation steps in supplementary materials.
\end{proof}

\begin{thm}($\nabla J(w)$ for Emphatic-TD($\lambda$) Solution) Regardless of on-policy or off-policy or the value of $\lambda$,  we have
\begin{eqnarray}
\nabla J(w) = \lim_{t \rightarrow \infty} \mathbb{E}[ \rho_t  \delta_t \psi_t ],
\end{eqnarray}
where  $\psi_t = f^\lambda_t \nabla \log{\pi_w(a_t|s_t)}+z_t + \gamma \rho_{t-1} \psi_{t-1}$, $f^\lambda_t=m_t+\gamma \lambda \rho_{t-1} f^\lambda_{t-1}$,  $m_t  =1+\gamma \rho_{t-1} (m_{t-1}-\lambda) $, $z_t =\gamma \rho_{t-1}\left( \left[ (m_{t-1}-\lambda)\nabla \log{\pi_w(a_{t-1}|s_{t-1})} \right]+ z_{t-1}\right)$, $m_{-1}=\lambda$, $z_{-1}=0$.
\end{thm}
\begin{proof}
The derivations follows the derivations Theorem (\ref{thm:J_lambda1}) using the value of $f_t^\lambda$ in Lemma (3). 
\end{proof}
Thus the actor can have the following update:
\[
w_{t+1}=w_t+\beta_t \rho_t \delta_t \psi_t,
\]
where $\psi_t$ is defined in Theorem (4). In addition, the proof of convergence exactly follows Theorem (3), as such we avoid repetitions.

\section{Counterexample for Off-PAC}
 Off-PAC in Degris et~.al (2012), is not based on a true gradient direction of $J(w)$ and as such may not optimize the value functions. Here we provide a simple counterexample for the case of $\lambda=0$ for Off-PAC and compare it with Gradient-AC.  For the counterexample, let us consider an MDP with two states and two actions whose reward function for the state $1$ is $0$ and for the state $2$ is $1$. The environment is deterministic and if the agent takes action $a=1$ it moves to state $s=2$ and for taking action $a=2$ it moves to state $s=1$. Here, $\gamma=0.99$, $\p(s=1)=1$ and  $\p(s=2)=2$ (Also see $\theta \rightarrow 2\theta$ problem in Sutton et~al., 2016). The optimal policy is to take action 1. Let us initialize the target policies with the optimal policy. GTD(0) solution, for this case is $\theta=-2/(4\gamma-3)$, which is negative and GTD(1) solution is $\theta=1/(2(1-\gamma))$ which is positive. The general form of update for taking action $1$ from state $1$ is, $(1+2\gamma \theta-\theta) \nabla \log{\pi_w(a=1|s=1)}$, however, we see that by substituting the $\theta$ values, Off-PAC will reduce the probability of taking action $1$ while Gradient-GTD(1) aims to increase it. Thus Off-PAC goes against the direction of $J(w)$. 
\section{Conclusions and Future Works}
We proposed the first class of convergent Actor-Critic methods that work with off-policy learning and function approximation. Our algorithms are suitable to use for large-scale problems: they are online, incremental and have $O(n)$ complexity, in terms of per-time-step computation and memory, without introducing any new hyper-parameters. We focused primarily on problems where the action is continuous or large so that action-value functions would not be suitable to use due to the curse-of-dimensionality. Thus, our algorithms (just like classical Actor-Critic methods)  won't require to have a direct representation for actions. For example, consider a policy with underlying Normal (or log-normal) distribution, $\pi_w(a|s) \sim N(\mu_w(s),\sigma^2)$, where $\mu(s)=w^\top \phi(s)$. By learning the parameterized mean, $\mu_w(s)$ (and variance), one can find a (near-) optimal policy. Such types of policies can be used in various problems, including robotics or bidding strategies in Ad Tech., etc. We also showed that our Actor-Critic algorithms become just like classical Actor-Critic methods for the problem of on-policy, so they are a direct generalization. 

In addition, briefly we discussed about the solution quality of MSE in reinforcement learning, and our preliminary result shows that GTD(1)/TD(1) can provide a good quality solution with a simple normalization of eligibility traces . If we can reduce the variance by using $\lambda=1$, then it implies: 1) No $\lambda$ hyper-parameter tuning is needed; 2) We get one single algorithm for value function approximation, since GTD(1) and Emphatic-TD(1) become identical and their solution becomes MSE-Solution, which is easy to understand, 3) We get a single convergent Actor-Critic algorithm. 

Future works include, massive empirical studies, studies on bias-variance trade-offs, and potentially proposing new interesting ideas for variance reduction methods in reinforcement learning with classical Actor-Critic architecture.
\newpage
\section*{References}
\medskip
\def\hangin{\hangindent=0.2in}
\hangin
Barto, A. G., Sutton, R. S., \& Anderson, C. (1983). Neuron-like elements that can solve difficult learning control problems. {\em IEEE Transaction on Systems, Man and Cybernetics}, 13, 835Ð846.

\hangin
Bhatnagar, S., Sutton, R. S., Ghavamzadeh, M., Lee, M. (2009). Natural Actor-Critic Algorithms. Automatica.

\hangin
Borkar, V. S. (1997). Stochastic approximation with two timescales. {\em Systems and Control
Letters} 29:291-294.

\hangin
Borkar, V. S. (2008). Stochastic Approximation: A Dynamical Systems Viewpoint. Cambridge University Press.

\hangin

Degris, T., White, M. and Sutton, R.~S. (2012). Off-Policy
Actor-Critic. \textit{Proceedings of the 29th International Conference on Machine Learning},  Edinburgh, Scotland, UK.

\hangin

Gu, S., Lillicrap, T., Ghahramani, Z., Turner, R. E., and
Levine, S. (2017). Q-prop: Sample-efficient policy gradient with
an off-policy critic.  \textit{International Conference on Learning Representations}.

\hangin
Kushner, H. J. \& Yin, G. G. (2003). Stochastic Approximation Algorithms and Applica-
tions. Second Edition, Springer-Verlag.

\hangin
Konda,  V. R., and  Tsitsiklis, J. N. (2003)  "Actor-Critic Algorithms" , {\em SIAM Journal on Control and Optimization}, Vol. 42, No. 4, pp. 1143-1166.

\hangin
Maei, H. R., Szepesv\'ari, Cs., Bhatnagar, S., Sutton, R. S. (2010). Toward off-policy learning control with function approximation. {\em Proceedings of the 27th International Conference on Machine Learning}, pp.~ 719--726.

\hangin
Maei, H. R. (2011) Gradient temporal-difference learning algorithms.
PhD dissertation, University of Alberta.

\hangin

Precup, D., Sutton, R.~S. and Dasgupta, S. (2001).
Off-policy temporal-difference learning with function approximation.  {\em Proceedings of the 18th International Conference on Machine Learning}, pp.~417--424.

\hangin
Sutton, R.S. (1984). Temporal credit assignment in reinforcement learning (106 Mbytes).  Ph.D. dissertation, Department of Computer Science, University of Massachusetts, Amherst, MA 01003. Published as COINS Technical Report 84-2.

\hangin

Sutton, R.~S. (1988).
\newblock Learning to predict by the method of temporal differences.
\newblock {\em Machine Learning 3}:9--44.

\hangin

Sutton, R.~S., Barto, A.~G. (1998).
{\em Reinforcement Learning: An Introduction}.
MIT Press.

\hangin

Sutton, R.~S., Maei, H.~R, Precup, D., Bhatnagar, S., Silver, D., Szepesv\'ari, Cs. \& Wiewiora, E. (2009). 
\newblock Fast gradient-descent methods for temporal-difference learning with linear function approximation.  
\newblock
In \textit{Proceedings of the 26th International Conference on Machine Learning}, pp. 993--1000.  Omnipress.

\hangin

Sutton, R. S., McAllester, D. A., Singh, S. P., and Mansour, Y. (1999). Policy gradient methods for
reinforcement learning with function approximation. In \textit{Advances in Neural Information Processing
Systems}, pages 1057-1063.

\hangin
Sutton, R. S., Mahmood, A. R., White, M. (2016). An emphatic approach to the problem of off-policy temporal-difference learning. {\em Journal of Machine Learning Research} 17(73):1--29.
\hangin

Tang, J. and Abbeel, P. (2010) On a connection between
importance sampling and the likelihood ratio policy
gradient. In \textit{Advances in Neural Information Processing
Systems}.

\hangin

Williams, R. J. (1987). A class of gradient-estimating algorithms for reinforcement learning in neural networks. In \textit{Proceedings of the IEEE First International Conference on Neural Networks}, San Diego, CA.



\nocite{langley00}

\bibliography{example_paper}
\bibliographystyle{icml2018}


\newpage



\section*{Supplementary material}

# Supplementary Materials:
# Convergent Actor-Critic Algorithms Under Off-Policy Training and Function Approximation

We present two key identities that would help to infer the constant value of $\eta$ vector as well as $f(s)$, $f^\lambda(s)$. Once we know these, the rest will become straightforward.

One key identity is, computing $Df$ for the case of GTD(1). The reason it is important because $\sum_s d(s)f(s)\cdot$, appears when write the value of expectation terms in terms of probabilities. For example, $\mathbb{E}[f_t \phi_t] = \sum_s d(s)f(s)\phi(s) = \Phi^\top Df\mathbf{1}$, where $\Phi$ is feature matrix whose $s^{\text{th}}$ row is $\phi(s)^\top$.

Another key identity is for the case of Emphatic-TD($\lambda$). There, $Df^\lambda$ is the important term to know, which we show is equal to $Df$.

**For the case of GTD(1) Solution:** The first thing to note is that the the weight $f(s)d(s)$ appears in many expectation terms as such it is worth to note the form of $Df$ vector, where $D$ is a diagonal matrix with diagonal elements $d$ and $f$ is a vector defined in the text.

We show $Df = (I - \gamma P^{\pi\top})^{-1}d$, where $f \in \mathbb{R}^{|\mathcal{S}|}$, $f(s) = \lim_t \mathbb{E}[f_t | s_t = s]$, and $f_t = 1 + \gamma \rho_{t-1} f_{t-1}$.

**Proof:** Let us denote Prob as probability function, then we have

$$f(s) = \lim_{t \to \infty} \mathbb{E}[f_t | s_t = s]$$

$$= 1 + \gamma \lim_{t \to \infty} \mathbb{E}[\rho_{t-1} f_{t-1} | s_t = s]$$

$$= 1 + \gamma \lim_{t \to \infty} \sum_{s^-, a^-} \text{Prob}(s_{t-1} = s^-, a^- | s)$$

$$\frac{\pi(a^-|s^-)}{\pi_b(a^-|s^-)} \mathbb{E}[f_{t-1} | s_{t-1} = s^-]$$

$$= 1 + \gamma \sum_{s^-, a^-} \frac{\text{Prob}(s|s^-, a^-)\text{Prob}(s^-, a^-)}{d(s)} \frac{\pi(a^-|s^-)}{\pi_b(a^-|s^-)}$$

$$\lim_{t \to \infty} \mathbb{E}[f_{t-1} | s_{t-1} = s^-]$$

$$= 1 + \gamma \sum_{s^-, a^-} \frac{P(s|s^-, a^-)\pi_b(a^-|s^-)d(s^-)}{d(s)} \frac{\pi(a^-|s^-)}{\pi_b(a^-|s^-)} f(s^-)$$

$$= 1 + \gamma \sum_{s^-} \frac{P^\pi(s|s^-)d(s^-)}{d(s)} f(s^-),$$

where we have used conditional probability $P(A|B) = P(B|A)/P(B)$, and the fact that $P^\pi(s'|s) = \sum_a \pi(a|s)P(s'|s,a)$, where $P(s'|s,a)$ referes to MDP transition probability model.

From the above equality easily we can get, $d(s)f(s) = d(s) + \gamma \sum_{s^-} P^\pi(s|s^-)d(s^-)f(s^-)$, and then $Df = (I - \gamma P^{\pi\top})^{-1}d$.

**For the case of Emphatic-TD($\lambda$):** Now for the case of emphatic-TD($\lambda$) we observe we have $Df^\lambda$ to compute. Just similar to the proof of $GTD(1)$, we assume $\eta(\lambda)$ is equal to a vector with all the elements zero except the last one that is 1. We go with this assumption and see that this solution exisit and satisfies in $\eta(\lambda)$ equation and since it is unique, it must be the solution. From $\eta$ solution we can also compute what $f^\lambda$ is:

$$Df^\lambda = (I - \gamma P^{\pi\top})^{-1}Dm$$

where $m(s) = \lim_{t \to \infty} \mathbb{E}[m_t | s_t = s]$. Now one can show (for more derivation details see Sutton et al., 2015)

We have $d(s)m(s) = \lambda d(s) + (1-\lambda)d(s)f(s)$, for all $s$. Now by turning it into matrix-vector product we have:

$$Dm = \lambda d + (1-\lambda)Df,$$

Now using the equality for $Df$, we have

$$Dm = \left[\lambda I + (1-\lambda)(I - \gamma P^{\pi\top})^{-1}\right]d.$$

We can show

$$\lambda I + (1-\lambda)(I - \gamma P^{\pi\top})^{-1} = (I - \gamma\lambda P^{\pi\top})(I - \gamma P^{\pi\top})^{-1}.$$

Thus, we get

$$Dm = (I - \gamma\lambda P^{\pi\top})(I - \gamma P^{\pi\top})^{-1}d.$$

Now we have:

$$Df^\lambda = (I - \gamma\lambda P^{\pi\top})^{-1}(I - \gamma\lambda P^{\pi\top})(I - \gamma P^{\pi\top})^{-1}d,$$

which implies $Df^\lambda = (I - \gamma P^{\pi\top})^{-1}d = Df$. Now similar to derivations we did for the case of GTD(1) we can prove Lemma 3.


\end{document}